\documentclass[letterpaper, 10 pt, conference]{ieeeconf}  % Comment this line out
\IEEEoverridecommandlockouts                              % This command is only
\title{\LARGE \bf A Formal Methods Approach to Pattern Synthesis\\ in Reaction Diffusion Systems
}

\author{Ebru Aydin Gol, Ezio Bartocci and Calin Belta% <-this % stops a space
\thanks{Ebru Aydin Gol (ebru@bu.edu) and Calin Belta (cbelta@bu.edu) are with Boston University.
Ezio  Bartocci (ezio.bartocci@tuwien.ac.at) is with Vienna University of Technology.}
}

\usepackage{times}
\usepackage{epsfig}
\usepackage{ifpdf}
\usepackage{stackrel}

\usepackage{array}
\usepackage{amsmath,amssymb,amscd,amstext}
\usepackage{mathrsfs}
\usepackage{dsfont}
\usepackage{amsbsy}
\usepackage{stmaryrd}

\usepackage{subfigure}

\usepackage{setspace}
\usepackage{fancyhdr} 
\usepackage{lastpage}
\usepackage{pslatex}

\usepackage[displaymath,textmath,sections,graphics,floats,graphicx,color]{preview}
\def  \ft#1     {\footnote{#1} }

\usepackage{tikz}

\usepackage{pdflscape}

\usepackage{algorithm, algpseudocode}

\usetikzlibrary{arrows,positioning,automata}

\newtheorem{remark}{\bf Remark}[section]
\newtheorem{proposition}{\bf Proposition}[section]

\newtheorem{definition}{\bf Definition}[section]
\newtheorem{example}{\bf Example}[section]
\newtheorem{theorem}{\bf Theorem}[section]
\newtheorem{problem}{\bf Problem}[section]

\renewcommand{\theequation}{\thesection.\arabic{equation}}

\renewcommand{\vec}[1]{\mathbf{#1}}
\renewcommand{\cal}[1]{\mathcal{#1}}
\newcommand{\bb}[1]{\mathbb{#1}}
\renewcommand{\phi}{\varphi}

\begin{document}

% GENERIC MATH MACROS
%\newcommand{\eventually}{\Diamond}
%\newcommand{\henceforth}{\Box}
\newcommand{\tuple}[1]{\langle #1\rangle}
\newcommand{\sat}{\models}
% \pr(X) = probability of X.  it's probably a good idea to use same font
% for \pr and \Peek.
\newcommand{\pr}{{\it P}}
\newcommand{\Peek}{{\it Peek}}
\newcommand{\comment}[1]{}

% \fullOnly{material that appears in tech report but not conference paper}
\newcommand{\fullOnly}[1]{}
%\newcommand{\fullOnly}[1]{#1}

% \confOnly{material that appears in conference paper but not tech report}
%\newcommand{\confOnly}[1]{}
\newcommand{\confOnly}[1]{#1}

% DOCUMENT-SPECIFIC MACROS
% observation symbol representing a gap
\newcommand{\gap}{{\it gap}}
% random variable representing the result of a peek operation.
% use regular font not \mathbf, i.e., assume a single variable,
% otherwise it looks odd, because X, O, and S are single variables.
\newcommand{\pkvar}{Q}
% meta-variable representing a value returned by a peek operation,
% i.e., a value of \pkvar.
\newcommand{\pk}{q}

% \sampled{O} = sampled version of trace O
%\newcommand{\samp}[1]{\check{#1}}
%\newcommand{\samp}[1]{{\rm samp}(#1)}

%%% Local Variables: 
%%% mode: latex
%%% TeX-master: "main"
%%% End: 

\maketitle
\thispagestyle{empty}
\pagestyle{empty}

%%%%%%%%%%%%%%%%%%%%%%%%%%%%%%%%%%%%%%%%%%%%%%%%%%%%%%%%%%%%%%%%%%%%%%%%%%%%%%%%
\begin{abstract}
We propose a technique to detect and generate patterns in a network 
of locally interacting dynamical systems. Central to our approach is a 
novel spatial superposition logic, whose semantics is defined over the 
quad-tree of a partitioned image. We show that formulas in this logic can 
be efficiently learned from positive and negative examples of several 
types of patterns. 
We also demonstrate that pattern detection, which is implemented as 
a model checking algorithm, performs very well for test data sets 
different from the learning sets. 
We define a quantitative semantics for the logic and integrate the model 
checking algorithm with particle swarm optimization in a computational 
framework for synthesis of parameters leading to desired patterns in 
reaction-diffusion systems. 
\end{abstract}

%%%%%%%%%%%%%%%%%%%%%%%%%%%%%%%%%%%%%%%%%%%%%%%%%%%%%%%%%%%%%%%%%%%%%%%%%%%%%%%%
\section{INTRODUCTION}
\label{sec:introduction}

From the stripes of a zebra and the spots on a leopard to the
filaments (Anabaena) \cite{golden1998}, spirals, squares (Thiopedia rosea),
and vortex (Paenibacillus) \cite{scherrer1986} formed by single-cell organisms, patterns can
be found everywhere in nature. Pattern formation is at the very
origin of morphogenesis and developmental biology, and it is at the core of technologies such as
self-assembly, tissue engineering, and amorphous computing.
Even though it received a lot of attention from diverse communities such as
biology, computer science, and physics, the problem
of pattern formation is still not well understood. 

Pattern recognition is usually formulated as a machine learning problem \cite{CMB06}, in which
patterns are characterized either statistically \cite{JADRJM00} or
through a structural relationship among their features
\cite{PAVT80}. Despite its success in several application areas
\cite{VRHM99}, pattern recognition still lacks a formal
foundation. Can patterns be specified in a formal language with 
well-defined syntax and semantics? Can we develop algorithms 
for pattern detection from specification given in such a language? 
Given a large collection of locally interacting agents, can we design 
parameter synthesis rules, control and interaction strategies
guaranteeing the emergence of global patterns? In this paper, 
by drawing inspiration from model checking \cite{Emerson90,Clarke99}, we provide partial answers to these questions.  

We address the following problem: {\it Given a network of locally 
interacting dynamical systems, and given sets of positive and 
negative examples of a desired pattern, find parameter values 
that guarantee the occurrence of the pattern in the network at steady state.}  
Our approach is based on a novel spatial superposition logic, 
called {\it Tree Spatial Superposition Logic} (TSSL), 
whose semantics is defined over quad-trees of partitioned images. 
The decision of whether a pattern exists in an image becomes a 
model checking problem. A pattern descriptor is a TSSL formula, 
and we employ machine-learning techniques to infer such a formula 
from the given positive and negative examples of the pattern. 
To synthesize parameters of the original networked system leading 
to a desired pattern, we use a particle swarm optimization (PSO) 
algorithm. The optimization fitness function is given by a measure 
of satisfaction induced by the quantitative semantics that we introduce 
for the logic. We present examples showing that formulas in the proposed 
logic are good classifiers for some commonly encountered patterns.
While the overall algorithm can, in principle, be applied to any network 
of locally interacting systems, in this paper we focus on the Turing 
reaction-diffusion system~\cite{Turing1952}, and show that 
pattern-producing parameters can be automatically generated with our method. 

The rest of the paper is organized as follows.  In Section~\ref{sec:related} 
we discuss the work.  In Section~\ref{sec:system} we formulate the problem and outline our approach. We define the syntax 
and semantics of TSSL in Section~\ref{sec:logic}. A machine learning 
technique to learn TSSL formulas from positive and negative examples 
of desired patterns is developed in Section~\ref{sec:learning}. The solution 
to the pattern generation problem is presented in Section~\ref{sec:design} 
as a supervised, iterative procedure that integrates quantitative model 
checking and optimization. We conclude with final 
remarks and directions for future work in Section~\ref{sec:conclusion}.

%%%%%%%%%%%%%%%%%%%%%%%%%%%%%%%%%%%%%%%%%%%%%%%%%%%%%%%%%%%%%%%%%%%%%%%%%%%%%%%%
%%%%%%%%%%%%%%%%%%%%%%%%%%%%%%%%%%%%%%%%%%%%%%%%%%%%%%%%%%%%%%%%%%%%%%%%%%%%%%%%

\section{RELATED WORK}
\label{sec:related}

\emph{Pattern recognition} is a well-established technique in machine 
learning. Given a data set and a set of classes, the goal is to 
assign each data to one class, or to provide a ``most likely" matching 
of the data to the classes. The two main steps in pattern recognition are: 
(a) to extract distinctive features~\cite{Julesz1981,Dalal2005,Belongie2002,Lowe1999} 
with relevant information from a set of 
input data representing the pattern of interest and (b) to build, using 
one of the several available machine learning techniques
(see~\cite{Russell} for a detailed survey),
an accurate classifier trained with the extracted features.
The descriptor chosen in \emph{feature extraction phase} 
depends on the application domain and the specific problem.  

This work is related to pattern recognition 
in \emph{computer vision}, where these descriptors may assume
different forms. Feature descriptors such as \emph{Textons}~\cite{Julesz1981} and
\emph{Histograms of Oriented Gradients} (HoG)~\cite{Dalal2005} 
are concerned with statistical information of color distribution or
of intensity gradients and edge directions. The \emph{scale-invariant feature
transform} (SIFT), proposed by Lowe in~\cite{Lowe1999},
is based on the appearance of an object at particular interest points, 
and is invariant to image scale and rotation. 
The \emph{shape context}~\cite{Belongie2002} is another feature descriptor 
intended to describe the shape of an object by the points of 
its contours and the surrounding context.

In this paper we establish an interesting connection between verification and pattern recognition. 
Both classical verification~\cite{Rizk2009,Donze2010,Donze2011,Fainekos2007a,Fainekos2009}
and pattern recognition techniques aim to verify (and possibly quantify) the emergence of a behavioral pattern. 
We propose logic formulas as pattern descriptors and verification techniques as pattern classifiers. 
The logical nature of such pattern descriptors allows to reason 
about the patterns and to infer interesting properties. 
For example, in~\cite{grosu2009learning}, the spatial 
modalities are used to characterize self-similar (fractal) texture.
Furthermore, combining different pattern descriptors
using both modal and logical operators is quite intuitive.

 This paper is inspired by the original work on morphogenesis by 
 Alan Turing~\cite{Turing1952}, and is closely related to~\cite{grosu2009learning}. 
 In the latter, the authors introduced a Linear Spatial Superposition Logic (LSSL), 
 whose formulas were interpreted, as in this paper, over quad-tree partitions. 
 The existence of a pattern in an image corresponded to the existence of a path 
 in the corresponding tree from the root to the leaf corresponding to a representative 
 point in the image. As a consequence, the method was shown to work for spirals, 
 for which the center was chosen as the representative point. The tree logic proposed 
 here is more general as it does not depend on the choice of such a point and captures 
 the pattern ``globally".  For example, the patterns considered in this paper
 cannot be expressed in LSSL, because they rely on a tree representation 
 rather than a path representation.
 
 As opposed to \cite{grosu2009learning}, we also define a 
 quantitative semantics for the logic, and use the distance to satisfaction as a fitness 
 function while searching for pattern-producing parameters. This quantitative semantics 
 and the discounted model checking on a computational tree are 
inspired from~\cite{Alfaro2009}, with the notable difference that we do not need a 
metric distance, but rather a measure of satisfiability. Such measures have also been used in~\cite{Rizk2009,Donze2010,Donze2011,Fainekos2007a,Fainekos2009}. 
The main novelty of this paper, compared to the other pattern recognition approaches, 
is that we can quantify ``how far" a system is from producing a desired pattern. This, which is possible due to the quantitative semantics of our logic, enables the use of optimization algorithms to search for pattern-producing parameters.

%%%%%%%%%%%%%%%%%%%%%%%%%%%%%%%%%%%%%%%%%%%%%%%%%%%%%%%%%%%%%%%%%%%%%%%%%%%%%%%%
%%%%%%%%%%%%%%%%%%%%%%%%%%%%%%%%%%%%%%%%%%%%%%%%%%%%%%%%%%%%%%%%%%%%%%%%%%%%%%%%
\newcommand{\paramSet}{\vec p}
\newcommand{\paramSetPattern}{\vec p^\ast}
\newcommand{\pattern}{+}
\newcommand{\nopattern}{-}

% system macros:

\newcommand{\systemIJ}{S_{i,j}}
\newcommand{\stateIJ}{\vec x_{i,j}}
\newcommand{\stateIJN}[1][n]{x_{i,j}^{(#1)}}
\newcommand{\inputIJ}{\vec u_{i,j}}
\newcommand{\inputIJN}[1][n]{u_{i,j}^{(#1)}}
\newcommand{\systemP}[1][{\vec p}]{\vec S^{(#1)}}

\section{PROBLEM FORMULATION}\label{sec:system}

{\emph{Notation.}} We use $\mathbb{R}$, $\mathbb{R}_+$, $\mathbb{N}$ and $\mathbb{N}_+$ to denote the set of real numbers, non-negative reals, integer numbers, and non-negative integers, respectively. For any $c \in \mathbb{R}$ and set $\mathcal{S} \subseteq \mathbb{R}$, $\mathcal{S}_{> c} := \{ x\in \mathcal{S} \mid x > c\}$, and for any $a,b \in \mathbb{R}$, $\mathcal{S}_{[a,b]} := \{ x \in \mathcal{S} \mid a\leq x \leq b\}$. 

A reaction-diffusion system $\vec S$ is modeled as a spatially distributed and locally interacting $K \times K$ rectangular grid of identical systems, where  each location $(i,j) \in \mathbb{N}_{[1,K]} \times \mathbb{N}_{[1,K]}$ corresponds to a system:
\begin{equation}\label{eq:rd_system}
	\systemIJ : \frac{d \stateIJN}{dt} = D_n (\inputIJN -\stateIJN) + f_n( \stateIJ, \vec R), \quad n=1,\ldots,N, 
\end{equation}
where $\stateIJ= [\stateIJN[1],\ldots, \stateIJN[N]]$ is the state vector of system $\systemIJ$, which captures the concentrations of all species of interest. $\vec D$ and $\vec  R$ are the parameters of system $\vec S$. $\vec D = [D_1,\ldots, D_N] \in \mathbb{R}^N_+$ is the vector of diffusion coefficients. $\vec R \in  \mathbb{R}^{P - N}$ is the vector of parameters that defines the local dynamics $f_n : \mathbb{R}^N_+ \times \mathbb{R}^{P - N} \rightarrow \mathbb{R}$ for each of the species $n=1,\ldots,N$.  Note that the parameters and dynamics are the same for all systems $\systemIJ, (i,j) \in \mathbb{N}_{[1,K]} \times \mathbb{N}_{[1,K]}$. The diffusion coefficient is strictly positive for diffusible species and it is $0$ for non-diffusible species. Finally, $\inputIJ= [\inputIJN[1],\ldots, \inputIJN[N]]$ is the input of system $\systemIJ$ from the neighboring systems:
\[ \inputIJN= \frac{1}{|\nu_{i,j}|} \sum_{v \in \nu_{i,j}} x^{(n)}_v,\]
where $\nu_{i,j}$ denotes the set of indices of systems adjacent to $\systemIJ$.

Given a parameter vector $\paramSet = [D, R] \in \mathbb{R}^{P}$, we use $\systemP[\paramSet]$ to denote an instantiation of a reaction-diffusion system. We use $\vec x(t) \in \mathbb{R}^{K\times K \times N}_+$ to denote the state of system $\systemP[\paramSet]$ at time $t$, and $\stateIJ(t) \in \mathbb{R}^N_+$ to denote the state of system $\systemIJ^{(\paramSet)}$ at time $t$.  
While the model captures the dynamics of concentrations of all species of interest, we assume that a subset $\{n_1,\ldots,n_o\} \subseteq \{1,\ldots,N\}$ of the species is observable through:
\[ H : \mathbb{R}^{K\times K \times N}_+ \rightarrow \mathbb{R}_{[0,b]}^{K\times K \times o} : \qquad \vec y = H(\vec x), \] 
for some $b \in \mathbb{R}_+$. For example, a subset of the genes in a gene network are tagged with fluorescent reporters. The relative concentrations of the corresponding proteins can be inferred by using fluorescence microscopy.  

We are interested in analyzing the observations generated by system~\eqref{eq:rd_system} in steady state. Therefore, we focus on parameters that generate steady state behavior, which can be easily checked through a running average:
\begin{equation}\label{eq:steady_state}
\sum_{i=1}^K \sum_{j=1}^K \sum_{n=1}^N \mid  \stateIJN(t) - \stateIJN \mid  < \epsilon,
\end{equation}
where $\stateIJN  =  \int_{t-T}^T \stateIJN(\tau) d\tau / T$ for some $T \leq t$. The system is said to be in steady state at time $\bar t$, if \eqref{eq:steady_state} holds for all $t \geq \bar t$.
In the rest of the paper, we will simply call the observation of a trajectory at steady state as the {\it observation} of the trajectory, and denote it as $H(\vec x(\bar t))$.

\begin{example} \label{ex:system} 
We consider a $32 \times 32$ reaction-diffusion system with two species ({\em i.e.} $K=32$, $N=2$):
\begin{align} \label{eq:rd_system_p_ex}
 	\frac{d\stateIJN[1]}{dt} &= D_1 \left (\inputIJN[1] - \stateIJN[1] \right) + R_1 \stateIJN[1] \stateIJN[2]- \stateIJN[1] + R_2, \nonumber \\
	\frac{d\stateIJN[2]}{dt} &= D_2 \left (\inputIJN[2] - \stateIJN[2] \right) + R_3 \stateIJN[1] \stateIJN[2] +  R_4 .
\end{align} 
The system is inspired from Turing's reaction-diffusion system and is presented in~\cite{systemApplet} as a model of the skin pigments of an animal. At a cell (location $(i,j)$), the concentration of species 1, $x^{(1)}_{i,j}$, depends on the concentration of species 1 in this cell and in its neighbors (if $D_1 > 0$), and the concentration of species 2 in this cell only, \emph{i.e.} $x^{(2)}_{i,j}$.   Similarly, $x^{(2)}_{i,j}$ depends on the concentration of species 2 in this cell and in its neighbors (if $D_2 > 0$), and $x^{(1)}_{i,j}$ (if $R_3 \neq 0$). 
We assume that species $1$ is observable through mapping $H : \mathbb{R}_+^{32\times32\times 2} \rightarrow \mathbb{R}_{[0,1]}^{32 \times 32}$ given by:
\vspace{-1mm}
 \[ \vec y  =  H(\vec x) , \text{ where } y_{i,j} = \frac { \vec x_{i,j}^{(1)}} { \max_{m,n} \vec x_{m,n}^{(1)} }.\]
We simulate the system from random initial conditions with parameters $\vec R = [1,-12,-1,16]$, and different diffusion parameters $\vec D_1 = [5.6, 24.5]$, $\vec D_2 = [0.2, 20]$, and $\vec D_3 = [1.4, 5.3]$. The observed concentrations of species 1 at different time points are shown in Figure~\ref{fig:exampleSnapshots}. At time $t=50$, all trajectories are in steady state. 
Note that, in all three cases, the spatial distribution of the steady state concentrations of species 1 has some regularity, \emph{i.e.} 
it forms a ``pattern". We will use {\it large spots} (LS), {\it fine patches} (FP), and {\it small spots} (SS) to refer   
to the patterns corresponding to $\vec D_1$, $\vec D_2$, and $\vec D_3$, respectively.  
\end{example}

\begin{figure}[h!]
   \includegraphics[width=0.5\textwidth,natwidth=8.5cm,natheight=6cm]{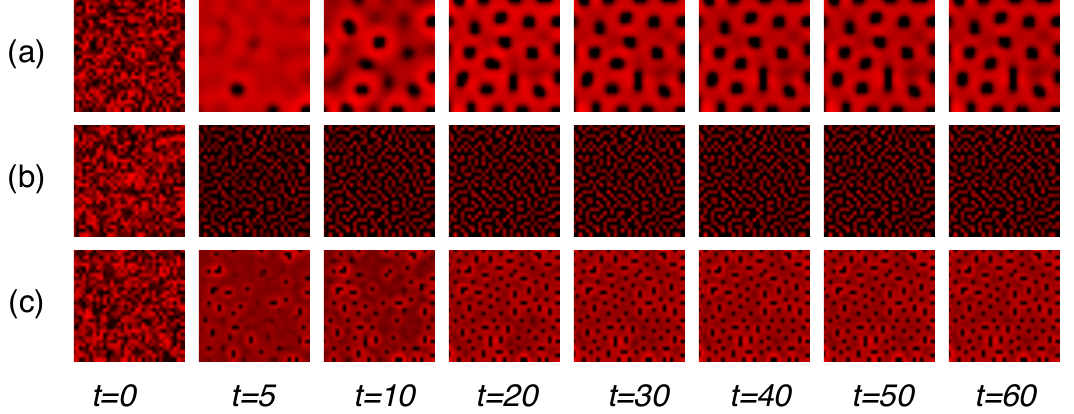} 
   \vspace{-5mm}
\caption{Observations generated by system~\eqref{eq:rd_system_p_ex} with parameters $\vec R$ and (a) $\vec D_1$, (b) $\vec D_2 $, and (c) $\vec D_3$ from Example~\ref{ex:system} (the  concentration of species 1 is represented with shades of red). The steady state observations produce (a) {\it large spots} (LS), (b) {\it fine patches} (FP), and (c) {\it small spots} (SS).}
\label{fig:exampleSnapshots}
\end{figure} % observed

\begin{problem}\label{prob:main}Given a reaction-diffusion system $\vec S$ as defined in~\eqref{eq:rd_system}, a finite set of initial conditions $\cal{X}_0 \subset \bb{R}^{K \times K \times N}$, ranges of the design parameters $\cal{P} = \cal P_1 \times \ldots \times \cal P_P,$ $\cal P_i \subset \mathbb{R},  i=1,\ldots,P$,  a set of observations $\vec Y_\pattern = \{\vec y_i \}_{i=1,\ldots,N_{\pattern}}$ that contain a desired pattern, a set of observations $\vec Y_{\nopattern}= \{\vec y_i \}_{i=1,\ldots,N_{\nopattern}}$ that do not contain the pattern, find parameters $\paramSetPattern \in \cal P$ such that the trajectories of system  $\systemP[\paramSetPattern]$ originating from $\cal{X}_0$ are guaranteed to produce observations similar to the ones from the set $\vec Y_{\pattern}$.
\end{problem}

\noindent To solve Problem~\ref{prob:main}, we need to perform two steps:
\begin{itemize}
	\item Design a mechanism that decides whether an observation contains a pattern.
	\item Develop a search algorithm over the state space of the design parameters to find $\paramSetPattern$.
\end{itemize}

The first step requires to define a pattern descriptor. To this goal, we develop a new spatial logic over spatial-superposition trees obtained from the observations, and treat the decision problem as a model checking problem. The new logic and the superposition trees are explained in Section \ref{sec:logic}. Then, finding a pattern descriptor reduces to finding a formula of the new logic that specifies the desired pattern. We employ machine-learning techniques to learn such a formula from the given sets of observations $\vec Y_{\pattern}$ and $\vec Y_{\nopattern}$. 

The second step is the synthesis of parameters $\paramSetPattern$ such that the observations produced by the corresponding reaction-diffusion system $\systemP[\paramSetPattern]$ satisfy the formula learned in the first step. To this end, we introduce quantitative semantics for the new logic, which assigns a positive valuation only to the superposition-trees that satisfy the formula. This quantitative valuation is treated as a measure of satisfaction, and is used as the fitness function in a particle swarm optimization (PSO) algorithm. The choice of PSO is  motivated by its inherent distributed nature, and its ability to operate on irregular search spaces, \emph{i.e.} it does not require a differentiable fitness function.  
Finally, we propose a supervised, iterative procedure to find $\paramSetPattern$ that solves Problem~\ref{prob:main}. The procedure involves iterative applications of steps one and two, and an update of the set $\vec Y_{\nopattern}$ until a parameter set that solves Problem~\ref{prob:main} is found, which is decided by the user.  

%%%%%%%%%%%%%%%%%%%%%%%%%%%%%%%%%%%%%%%%%%%%%%%%%%%%%%%%%%%%%%%%%%%%%%%%%%%%%%%%
%%%%%%%%%%%%%%%%%%%%%%%%%%%%%%%%%%%%%%%%%%%%%%%%%%%%%%%%%%%%%%%%%%%%%%%%%%%%%%%%

\section{TREE SPATIAL SUPERPOSITION LOGIC}
\label{sec:logic}

\subsection{Quad-tree spatial representation}

We represent the observations of a reaction-diffusion system  as  a matrix
$\mathcal{A}_{k,k}$ of $2^k \times 2^k$ elements  $a_{i,j}$ with $k \in \mathbb{N}_{>0}$.
Each element corresponds to a small region in the space and 
 is defined as a tuple $a_{i,j} = \tuple{a^{(1)}_{i,j}, \cdots, a^{(o)}_{i,j}}$ of values 
representing the concentration of the observable species within an 
interval $a^{(c)}_{i,j}  \in [0, b]$, with  $b \in \mathbb{R}_{+}$.
Given a matrix $\mathcal{A}_{k,k}$,  we use $\mathcal{A}_{k,k}[i_s,i_e; j_s,j_e]$ to 
denote the sub-matrix formed by selecting the rows with indices from 
$i_s$ to $i_e$ and the columns with indices from 
$j_s$ to $j_e$.

 \begin{definition}\label{def1}  A quad-tree $Q = (V,R)$ is a quaternary 
 tree~\cite{Finkel1974} representation of $\mathcal{A}_{k,k}$ 
where each vertex $v \in V$ represents a sub-matrix of $\mathcal{A}_{k,k}$ and
the relation $R \subset V \times V$  defines the four children of each
node $v$ that is not a leaf. 
A vertex $v$ is a leaf when all the elements of the sub-matrix that it represents 
have the same values.
 \end{definition}
 
 Figure~\ref{fig:quadtree} shows an example of a quadtree,
where node $v_0$ represents the entire matrix; child $v_1$ represents the sub-matrix
$\{1, \cdots, 2^{k-1}\} \times \{1, \cdots, 2^{k-1}\} $; child $v_7$ represents 
the sub-matrix $\{2^{k-2}+1, \cdots, 2^{k-1}\}  \times \{2^{k-2}+1, \cdots, 2^{k-1}\}$; etc. In Figure~\ref{fig:quadtree}, we also label each edge in the quad-tree with the direction 
of the sub-matrix represented by the child: north west (NW), 
north east (NE), south west (SW), south east (SE).

 \begin{figure}[htbp]
   \centering
   \includegraphics[width=8.5cm,,natwidth=8.5cm,natheight=6cm]{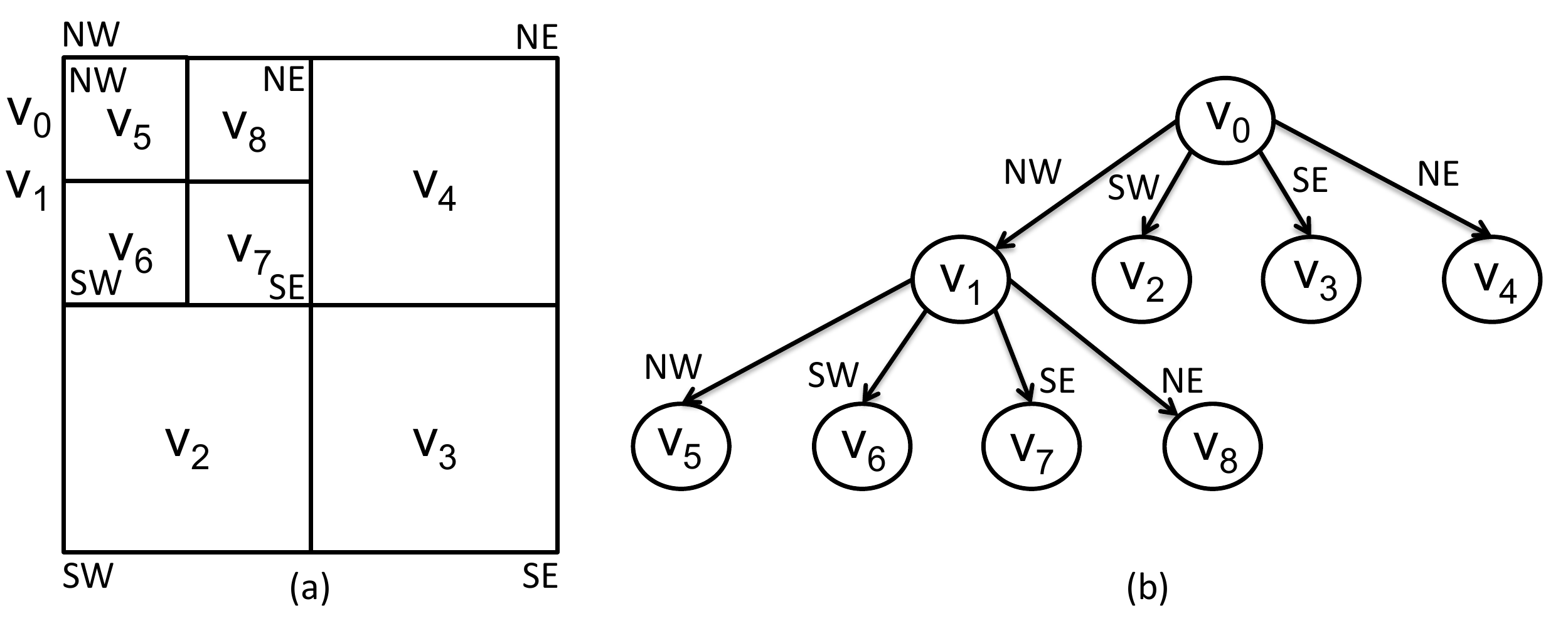} % requires the graphicx package
   \caption{Quad-tree representation (b) of a matrix (a).}
   \label{fig:quadtree}
\end{figure}
 
 \vspace{-5mm}
 
  \begin{definition}\label{def2}
  We define the mean  function $\mu_c : V \rightarrow [0,b]$ for 
 sub-matrix $\mathcal{A}_{k,k}[i_s,i_e; j_s,j_e]$
  represented by the vertex $v \in V$ of the quad-tree $Q = (V,R)$ as follows:
 $$\mu_c(v) = \frac{1}{(i_e - i_s + 1) (j_e-j_s +1)} \sum_{i,j \in \{i_s, \cdots, i_e\} \times  \{j_s, \cdots, j_e\}  } a^{(c)}_{i,j}$$

 \end{definition}
 
 \def\dom{\mathop{\rm dom}\nolimits}

\begin{table}[!ht]
\label{list:algorithm}
\begin{tabular*}{8.5cm}{rl}
\hline
\noalign{\vskip 0.3em}
\multicolumn{2}{l}{\textbf{Algorithm} \textsc{BuildingQuadTransitionSystem}} \\
\noalign{\vskip 0.3em}
\hline
\noalign{\vskip 0.3em}
\multicolumn{2}{l}{\textbf{Input}: {\hspace*{0.30cm}}\footnotesize{Matrix $\mathcal{A}_{k,k}$ of $2^k \times 2^k$ of elements $a_{i,j} = \tuple{a^{(1)}_{i,j}, \cdots, a^{(o)}_{i,j}}$,}}\\ 
\multicolumn{2}{l}{{\hspace*{1.36cm}}\footnotesize{its quad-tree  $Q\,{=}\,(V,R)$, the root $v_0 \in V$, and a labeling}}\\ 

\multicolumn{2}{l}{ {\hspace*{1.35cm}}\footnotesize{function $LQ : R \rightarrow \mathcal{D}=\{NW, NE, SE, SW\}$}}\\ 
\multicolumn{2}{l}{\textbf{Output}:\hspace*{0.05cm}~\footnotesize{Quad Transition System $\mathcal{Q}_{TS} = ( S, s_\iota, \tau, \Sigma, [.] , L)$}} \\
\noalign{\vskip 0.5em}
 1: & \verb| |\footnotesize{$\Sigma := \{ m_1, \cdots, m_o \}$} \hspace*{0.1cm}  $\triangleright$ \scriptsize{Initialize the set of variables $\Sigma$ of $\mathcal{Q}_{TS}$}.  \\
 2: & \verb| |\footnotesize{$\tau= \emptyset$}  \hspace*{0.1cm} $\triangleright$ \scriptsize{Initialize the set $\tau$ of the transition relation $\tau$ of $\mathcal{Q}_{TS}$}. \\
 3: & \verb| |\footnotesize{$S := \{s_{\iota} \} $}  \hspace*{0.1cm} $\triangleright$ \scriptsize{Initialize the set of states $S$ of $\mathcal{Q}_{TS}$}.  \\
 4: & \verb| |\footnotesize{$TS := \{ \tuple{s_{\iota}, \{ v_0 \} } \} $} \hspace*{0.1cm}  \\
  \multicolumn{2}{l}{ {\hspace*{1.0cm}} $\triangleright$  \scriptsize{Each tuple in TS contains a state in S and a set of vertices in V.}}\\ 
  
 5: & \verb| |\footnotesize{$LF := \{ v \in V | \not \exists t \in V: (v,t) \in R\}$} \hspace*{0.10cm} $\triangleright$  \scriptsize{LF is the set of leaves of Q}  \\
 6: & \verb| |\footnotesize{$PLF := \{ P_i \subseteq LF, 1 \leq i \leq n$ $|$ $P_i \neq \emptyset \wedge \forall v_a,v_b \in P_i, $} \\
 \multicolumn{2}{l}{ {\hspace*{4.5cm}}\footnotesize{$\forall v_c \in P_{j \neq i}, v_a \equiv v_b \wedge v_a \not \equiv v_c \}$}}\\ 
   \multicolumn{2}{l}{ {\hspace*{1.0cm}} $\triangleright$  \scriptsize{PLF is a partition of LF with equivalent leaves. }}\\ 
   
 7:& \verb| |\footnotesize{\textbf{for each} $\hat{P} \in PLF $ \textbf{do}}  \hspace*{0.1cm} \\
  \multicolumn{2}{l}{ {\hspace*{1.0cm}} $\triangleright$  \scriptsize{For each partition element, create a state $s'$ with a self-loop and}}\\ 
\multicolumn{2}{l}{ {\hspace*{1.0cm}} $\triangleright$  \scriptsize{a transition to the state $s_{\iota}$ if $\hat{P}$ contains a child of $v_0$.}}\\ 
8:& \verb|   |\footnotesize{add new state $s'$ to $S$ and a tuple $\tuple{s',\hat{P}}$ to $TS$}\\
9:&  \verb|   |\footnotesize{$\tau := \tau \cup \{ (s',s') \} \cup \{(s,s') : \tuple{s, VS} \in TS,$}\\ 
 \multicolumn{2}{l}{ {\hspace*{4.5cm}}\footnotesize{ $\exists v \in VS, \exists v' \in \hat{P}: (v,v') \in R \}$}}\\ 

10:& \verb| |\footnotesize{\textbf{end for}} \\
11: & \verb| |\footnotesize{$FS := \{ v \in V | (v_0,v) \in R \} \backslash LF$}   \hspace*{0.1cm}  \\
 \multicolumn{2}{l}{ {\hspace*{1.0cm}} $\triangleright$  \scriptsize{explore the children of $v_0$ that are not leaves.}}\\ 

12: & \verb| |\footnotesize{\textbf{while} $FS \neq \emptyset $ \textbf{do} \hspace*{0.1cm}}  $\triangleright$  \scriptsize{FS contains the frontier vertices to be explored.} \\
13: & \verb|   |\footnotesize{$LFS := \{ v \in FS$ $|$ $\forall v' \in V : (v,v') \in R: $} \\
 \multicolumn{2}{l}{ {\hspace*{4.5cm}}\footnotesize{ $\exists \tuple{s,VS} \in TS \wedge v' \in VS \}$}}\\

14: & \verb|   |\footnotesize{$PLFS := \{ P_{i \in I} \subseteq LFS$ $|$ $I \neq \emptyset, P_i \neq \emptyset, \forall v_a,v_b \in P_i,$} \\
 \multicolumn{2}{l}{ {\hspace*{4.1cm}}\footnotesize{ $ \forall v_c \in P_{j \neq i}, v_a \equiv v_b \wedge v_a \not \equiv v_c \}$}}\\ 

 15:& \verb|   |\footnotesize{\textbf{for each} $\hat{P} \in PLFS $ \textbf{do}} \\
 16:& \verb|     |\footnotesize{add new state $s'$ to $S$ and a tuple $\tuple{s',\hat{P}}$ to $TS$}\\
 17:&  \verb|     |\footnotesize{$\tau:=  (\bigcup_{s : \tuple{s,VS} \in TS: \exists v \in \hat{P},  \exists v' \in  VS, (v,v') \in R} (s',s)) \cup \tau $}\\
18:&  \verb|     |\footnotesize{\textbf{if} $\exists v \in \hat{P} \wedge \exists \tuple{s,VS}: \exists v' \in VS \wedge (v', v) \in R$ \textbf{then}}\\
19:&  \verb|       |\footnotesize{$\tau := \tau \cup \{ (s,s') \} $}\\
20:&  \verb|     |\footnotesize{\textbf{end if}} \\
21:& \verb|   |\footnotesize{\textbf{end for}} \\
 22:& \verb|   |\footnotesize{\textbf{for each} $\hat{v} \in FS\backslash LFS $ \textbf{do} }\\
  23:& \verb|     |\footnotesize{add new state $s'$ to S and a tuple $\tuple{s',\{ \hat{v} \}}$ to TS}\\
   24:&  \verb|     |\footnotesize{$\tau:=  (\bigcup_{s : \tuple{s,VS} \in TS:  \exists v' \in  VS, (\hat{v},v') \in R} (s',s)) \cup \tau $}\\
25:&  \verb|     |\footnotesize{\textbf{if} $\exists \tuple{s,VS}: \exists v' \in VS \wedge (v', \hat{v}) \in R$ \textbf{then}}\\
26:&  \verb|       |\footnotesize{$\tau := \tau \cup \{ (s,s') \} $}\\
27:&  \verb|     |\footnotesize{\textbf{end if}} \\
28:& \verb|   |\footnotesize{\textbf{end for}} \\
29: & \verb|   | \footnotesize{$FS := \{ v \in V $ $|$ $\exists \bar{v} \in FS, (\bar{v},v) \in R \} \backslash LF $ }  \\
30: & \verb| |\footnotesize{\textbf{end while}} \\
31: & \verb| |\footnotesize{\textbf{define func} $[.]$ as $[\bar{c}](\bar{s}) := \mu_{\bar{c}}(v_{\bar{s}})$, $\bar{c} \in \{1,\cdots,o\}, $}  \\
 \multicolumn{2}{l}{ {\hspace*{4.1cm}}\footnotesize{ $v_{\bar{s}} \in VS: \tuple{\bar{s},VS} \in TS$}}\\ 
32: & \verb| |\footnotesize{\textbf{define func} $L$ as $L(s,t) := (t = s) ? \mathcal{D} :  $}\\
 \multicolumn{2}{l}{ {\hspace*{3.4cm}}\footnotesize{ $\bigcup_{\tilde{v} \in \tilde{VS}, \bar{v} \in \bar{VT}    :\tuple{s,\tilde{VS}},\tuple{t,\bar{VT}} \in TS, (\tilde{v}, \bar{v}) \in R } LQ(\tilde{v}, \bar{v})$}}\\ 

33: & \verb| |\footnotesize{\textbf{return} $S, s_\iota, \tau, \Sigma, [.] , L$}\\

\noalign{\vskip 0.3em}
\hline
\end{tabular*}
\end{table}
\label{sec:algorithm}
 
 \noindent The function $\mu_c$ provides the expected value for an observable variable with index  
 $c, 1 \leq c \leq o$ in a particular region of the space represented by the vertex $v$.

 \begin{definition}\label{def3}
Two vertices $v_a,v_b \in V$ are said to be equivalent when the mean function applied 
to the elements of the sub-matrices that they represent produce the same values:

$$v_a \equiv v_b  \Longleftrightarrow \mu_c(v_a) = \mu_c(v_b), \forall c, 1 \leq c \leq o$$

 \end{definition}

We use the mean of the concentration of the observable species  as a spatial 
abstraction (superposition) of  the observations in a particular region of the system, 
avoiding in this way to enumerate the observations of all locations. 
This approach is inspired 
by previous papers~\cite{grosu2009learning,Kwon2006}, where the authors aim to combat the 
state-explosion problem that would stem otherwise.

\begin{proposition}\label{prop1}
Given a vertex $v \in V$ of a quad-tree $Q = (V,R)$ and its four children $v_{NE}, v_{NW}, v_{SE}, v_{SW}$
the following property holds:
$$\mu_c(v) = \frac{\mu_c(v_{NE}) + \mu_c(v_{NW}) + \mu_c(v_{SE}) + \mu_c(v_{SW})}{4}$$ 
\end{proposition}

\begin{proof} The proof can be easily derived by expanding the terms of Definition~\ref{def2}.
\end{proof}

\begin{proposition}\label{prop2} 
The number of vertices needed for the quad-tree representation $Q = (V, R)$ of a matrix $\mathcal{A}_{k,k}$ is upper bounded by $\sum^{ k}_{i=0} 2^{2i}$.
\end{proposition}

\begin{proof} The proof follows from the fact that  the worst case scenario is when all the elements 
have different values. In this case the cardinality of the set $V$ is equal to the cardinality of a full and complete quaternary 
tree. 
For example, to represent the matrix $\mathcal{A}_{3,3}$,  it would require a max number of vertices $|V| \leq 1 + 4 + 16 + 64 = 85$. 
\end{proof}

 \begin{figure}[htbp]
   \centering
   \includegraphics[width=8.5cm,natwidth=8.5cm,natheight=6cm]{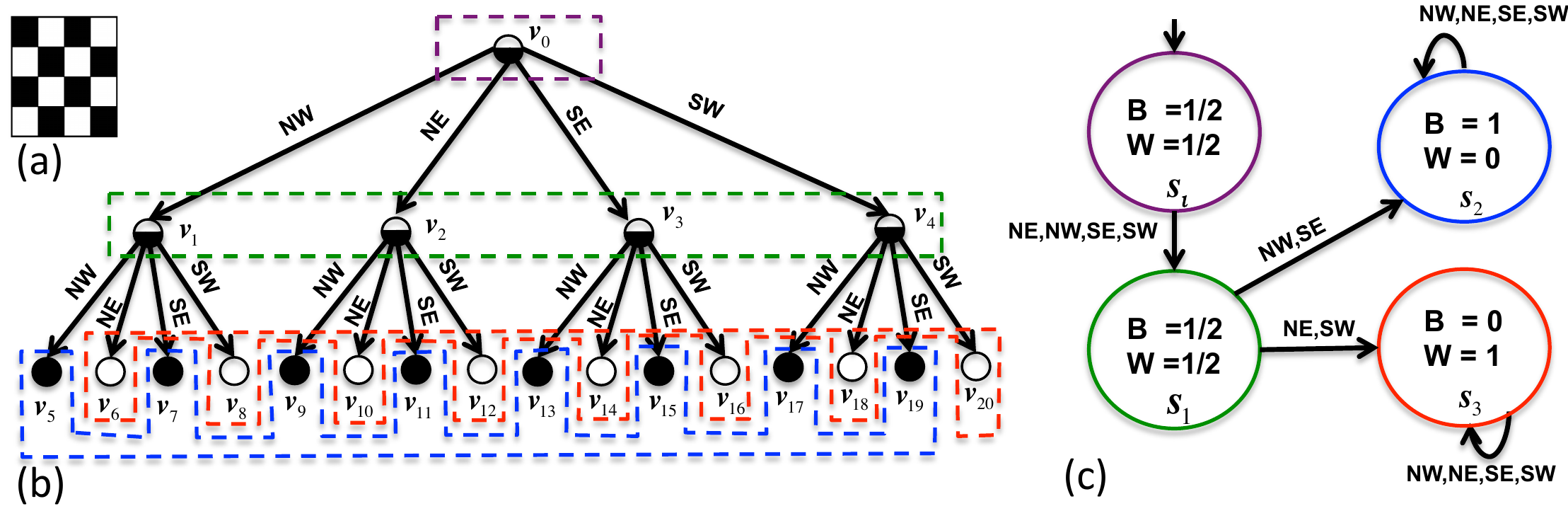} % requires the graphicx package
   \caption{A checkerboard pattern as a matrix of pixels (a), the quad-tree representation (b) and the derived quad transition system (c), where $\textbf{B}$ and $\textbf{W}$ denote black and white, respectively.}
   \label{fig:quadtree1}
\end{figure}

\subsection{Quad Transition System}

We now introduce the notion of quad transition system that extends the classical 
quad-tree structure, allowing for a more compact exploration for model checking.

\begin{definition} A \emph{Quad Transition System} (QTS) is a tuple $\mathcal{Q}_{TS} = ( S, s_\iota, \tau, \Sigma, [.] , L)$,
where:
\begin{enumerate}
\item $S$ is a finite set of states with $s_\iota \in S $ the initial state;
\item $\tau \subseteq S \times S $ is the transition relation. We require $\tau$ to be non-blocking and bounded-branching:
$\forall s \in S, \exists t \in S:  (s,t) \in \tau$ and $\forall s \in S$, if $T(s)=\{t : (s,t) \in \tau\}$ is the set of all successors
of $s$, the cardinality of  $|T(s)| \leq 4$;
\item $\Sigma$ is a finite set of variables;
\item $[.]$ is a function $[.] : S \rightarrow (\Sigma \rightarrow [0, b])$ that assigns to each state $s \in S$
and a variable $m \in \Sigma$ a rational value $[s](m)$ in $[0,b]$ with $b \in \mathbb{R}_{+}$;
\item $L$ is a labeling function for the transition $L: \tau \rightarrow 2^\mathcal{D}$ with $\mathcal{D}=\{NW, NE, SE, SW\}$ and with the 
 property that $\forall (s,t), (s,t') \in \tau$, with $t \neq t' $ it holds that $L(s,t) \cap L(s,t') = \emptyset $, $\bigcup_{\forall t \in S: (s,t) \in \tau} L (s,t) = \mathcal{D}$.

\end{enumerate}
\end{definition}

The \textsc{BuildingQuadTransitionSystem} algorithm shows how to generate a QTS starting 
from a quad-tree representation $Q\,{=}\,(V,R)$ of a a matrix $\mathcal{A}_{k,k}$
and a labeling function  $LQ : R \rightarrow \mathcal{D}$.

\begin{proposition}\label{prop3} 
A \emph{quad transition system} (QTS) $\mathcal{Q}_{TS} = ( S, s_\iota, \tau, \Sigma, [.] , L)$ generated by the  \textsc{BuildingQuadTransitionSystem} algorithm
has always a least fixed point, that is $\exists s \in S : (s, s) \in \tau $.
\end{proposition}

\begin{proof} This property holds because 
the algorithm generates  a state with a self-loop transition for each partition of equivalent leaves in the quad-tree. 
\end{proof}

\begin{definition}[\textbf{Labeled paths}]  Given a set $B$ of labels representing the spatial directions, a \emph{labeled path} (lpath) of a QTS $\mathcal{Q}$ is
an infinite sequence $\pi^B= s_0 s_1 s_2 \cdots $ of states such that $(s_i,s_{i+1}) \in \tau$  $\wedge$ $L(s_i,s_{i+1}) \cap B \neq \emptyset$, $\forall i \in \mathbb{N}$.  Given a state $s$, we denote $LPaths^B(s)$ the set of all labeled paths starting in $s$, and with $\pi^B_i$
the $i$-th element of a path $\pi^B \in LPaths^B(s)$.
For example, in Figure~\ref{fig:quadtree1}, $LPaths^B(s_{\iota}) = \{ s_{\iota} s_1 s_2 s_2\cdots \} $ if $B=\{NW,SE\}$.   
\end{definition}
\vspace*{-3.0ex}
\subsection{TSSL Syntax and Semantics}

\begin{definition} [\bf TSSL syntax] The syntax of TSSL is defined as follows:
\footnotesize
$$\varphi ::= \top \:| \: \bot \: | \:  m \sim d \: | \: \neg \varphi \:|\: \varphi_{1} \wedge \varphi_{2} \:|\:  \exists_{B} \bigcirc \: \varphi \:| \forall_{B}\bigcirc \: \varphi \:|  \exists_{B}  \:\varphi_{1} \:  \mathcal{U}_k \: \varphi_{2}  \:| \forall_{B} \:\varphi_{1} \:  \mathcal{U}_k \: \varphi_{2}$$
\normalsize
\noindent with $ \sim \in \{ \: \leq \:, \: \geq \}$, $d \in [0,b]$, $b \in \mathbb{R}_{+}$, $k \in \mathbb{N}_{>0}$, $B \subseteq \mathcal{D} : B \neq \emptyset$, and $m \in \Sigma$, with $\Sigma$ the set of variables.
\normalsize
\end{definition}

From this basic syntax one can derive other two temporal operators:  the {\it exist eventually} 
operator $\exists_{B}  F_k $, the {\it forall eventually}  operator $\forall_{B}  F_k$, 
the {\it exist globally} 
operator $\exists_{B}  G_k $, and the {\it forall globally}  operator $\forall_{B}  G_k$
defined such that:

\footnotesize
\begin{align}
 \exists_{B} F_k \varphi := \exists_{B} \top \:  \mathcal{U}_k \:  \varphi  \:  \:  \:  \: \:  \:  \:  \:  \exists_{B} G_k \:  \varphi := \neg \forall_{B} F_k \neg \varphi.\nonumber \\
 \forall_{B} F_k \varphi := \forall_{B} \top \:  \mathcal{U}_k \:  \varphi  \:  \:  \:  \: \:  \:  \:  \:  \forall_{B} G_k \:  \varphi := \neg \exists_{B} F_k \neg \varphi. \nonumber 
\end{align}
\normalsize

\noindent 
The TSSL logic resembles the classic CTL logic~\cite{Clarke1982}, with the main 
difference that the \emph{next} and \emph{until} are not 
temporal, but spatial operators  meaning a change 
of resolution (or zoom in).  The set $B$  selects the 
spatial directions in which the operator is allowed to work 
and the parameter $k$ limits the \emph{until} to operate 
on a finite sequence of states. 
In the following we provide the TSSL qualitative semantics that,
given a spatial model and a formula representing the pattern 
to detect, provides a yes/no answer. 

\begin{definition} [\bf TSSL Qualitative Semantics] Let   $\mathcal{Q} = ( S, s_\iota, \tau, \Sigma, [.] , L)$ be a QTS,
Then, $\mathcal{Q}$ \emph{satisfies a TSSL formula $\varphi$},
written $\mathcal{Q} \models  \varphi$, if and only if $\mathcal{Q}, s_\iota \models \varphi$, where:
\scriptsize

$$
\begin{array}{*{10}c}
   \begin{array}{l}
  
  \mathcal{Q}, s  \models \top  \\
  \mathcal{Q}, s  \models  m \sim d \\
  \mathcal{Q}, s   \models \neg \varphi  \\
  \mathcal{Q}, s  \models  \varphi_{1} \wedge \varphi_{2}   \\
  \mathcal{Q}, s  \models  \exists_{B} \bigcirc  \: \varphi \\
  \mathcal{Q}, s   \models  \forall_{B} \bigcirc\: \varphi \\
  \mathcal{Q}, s   \models \exists_{B} \varphi_{1} \: \mathcal{U}_k \: \varphi_{2} \\
  \\
    \mathcal{Q}, s  \models \forall_{B} \varphi_{1} \: \mathcal{U}_k \: \varphi_{2} \\
\\
  \end{array} &
     \begin{array}{c}
   \mbox{and} \\
  \Leftrightarrow \\
  \Leftrightarrow \\
 \Leftrightarrow \\
   \Leftrightarrow \\
   \Leftrightarrow \\
\Leftrightarrow \\
\\
\Leftrightarrow \\
\\

  \end{array}   &
       \begin{array}{l}
   Q, s \not \models \bot \\
  \mbox{[}s\mbox{]}(m) \sim d\\
     \mathcal{Q}, s  \not \models \varphi \\
\mathcal{Q}, s  \models  \varphi_{1} \wedge \mathcal{Q}, s   \models  \varphi_{2} \\
\exists s' : (s,s') \in \tau \wedge L(s,s') \cap B \neq \emptyset   \wedge  \mathcal{Q}, s' \models \varphi \\
\forall s' : (s,s') \in \tau \wedge L(s,s') \cap B \neq \emptyset   \wedge  \mathcal{Q}, s' \models \varphi \\

\exists \pi^B \in LPaths^B(s) : \exists i, 0 < i \leq  k : \\
  \mbox{\quad \quad}(Q, \pi^B_i \models \varphi_{2})  \wedge (\forall j < i,  (Q, \pi_j \models \varphi_{1}))  \\

\forall \pi^B \in LPaths^B(s) : \exists i, 0 < i \leq  k : \\
\mbox{\quad \quad} (Q, \pi^B_i \models \varphi_{2})  \wedge (\forall j < i,  (Q, \pi_j \models \varphi_{1})) \\
  \end{array}
 \end{array}
$$

\end{definition}

\normalsize

\begin{example} \label{ex:design} \textbf{Checkerboard pattern.}
The checkerboard pattern from Fig~\ref{fig:quadtree1} a) can be characterized with the following TSSL formula ($B^*=\{SW,NE,NW,SE\}$):
\footnotesize{
$$ \forall_{B^*} \bigcirc (   \forall_{B^*} \bigcirc (( \forall_{\{SW,NE\}}  \bigcirc ( m \geq 1)) \wedge (\forall_{\{NW,SE\}}  \bigcirc ( m \leq 0))   )  ).    $$}
\normalsize
The ``eventually" operator can be used to define all the possible checkerboards of different sizes less or equal than $4^2$ as follows:

\footnotesize{
$$ \forall_{B^*} F_2 ( (\forall_{\{SW,NE\}}  \bigcirc ( m \geq 1)) \wedge (\forall_{\{NW,SE\}}  \bigcirc ( m \leq 0))   )    $$}
\end{example}

\noindent The qualitative semantics is useful to check if a 
given spatial model violates or satisfies a pattern expressed in TSSL.
However, it does not provide any information about how much 
the property is violated or satisfied.  This information may 
be useful to guide a simulation-based parameter exploration 
for pattern generation. For this reason we equip our logic 
also with a quantitative valuation that provides a measure 
of satisfiability in the same spirit of~\cite{Donze2011}.
Since the valuation of a TSSL formula with spatial 
operators requires to traverse and to compare regions of space 
at different resolution, we apply a discount factor 
of $\frac{1}{4}$ on the result  each time a transition is taken in QTS.

\begin{definition} [ \bf TSSL Quantitative Semantics] Let   $\mathcal{Q} = ( S, s_\iota, \tau, \Sigma, [.] , L)$ be a QTS. The quantitative  valuation $\llbracket \varphi \rrbracket :  S \rightarrow [-b,b]$ of a TSSL formula $\varphi$ is defined as follows:
\scriptsize
\begin{align}
  \llbracket \top\rrbracket (s) 						&= b\nonumber \\
 \llbracket \bot \rrbracket (s) 						&= -b\nonumber \\
 \llbracket m \sim d \rrbracket (s) 					&= (\sim \mbox{is} \geq ) \mbox{ ? }  ([m](s) - d) : (d - [m](s)) \nonumber \\
  \llbracket \neg \varphi \rrbracket (s) 				&=  -\llbracket \varphi \rrbracket (s)\nonumber \\
\llbracket \varphi_1 \wedge \varphi_2 \rrbracket (s)	&= \min (\llbracket \varphi_1 \rrbracket (s), \llbracket \varphi_2 \rrbracket (s))  \nonumber   \\
 \llbracket \exists_{B} \bigcirc  \: \varphi \rrbracket (s)	&= \frac{1}{4} \max_{\scriptsize{\pi^B \in LPaths^B(s)}} \llbracket  \: \varphi \rrbracket (\pi^B_1)\nonumber \\
 \llbracket \forall_{B} \bigcirc  \: \varphi \rrbracket (s) 	&=  \frac{1}{4} \min_{\scriptsize{\pi^B \in LPaths^B(s)}}  \llbracket  \: \varphi \rrbracket (\pi^B_1)\nonumber \\
   \llbracket \exists_{B} \varphi_{1} \: \mathcal{U}_k \: \varphi_{2} \rrbracket (s) &= \sup_{\scriptsize{\pi^B \in LPaths^B(s)}}   \{ \min(\frac{1}{4^i}\llbracket  \varphi_2 \rrbracket (\pi^B_i), \inf \{\frac{1}{4^j} \llbracket  \varphi_1 \rrbracket (\pi^B_j) \mid  j < i \}) \mid 0 < i \leq k \}\} \nonumber \\
\llbracket \forall_{B} \varphi_{1} \: \mathcal{U}_k \: \varphi_{2} \rrbracket (s) &=\inf_{\scriptsize{\pi^B \in LPaths^B(s)}}  \{ \min(\frac{1}{4^i}\llbracket  \varphi_2 \rrbracket (\pi^B_i), \inf \{\frac{1}{4^j} \llbracket  \varphi_1 \rrbracket (\pi^B_j) \mid j < i  \}) \mid 0 < i \leq k  \} \nonumber
\end{align}
\end{definition}
\normalsize

\begin{theorem}[\textbf{Soundness}]\label{theorem1}  Let   $\mathcal{Q} = ( S, s_\iota, \tau, \Sigma, [.] , L)$ be a QTS, $s \in S$ a state of $\mathcal{Q}$, 
and $\varphi$ a TSSL formula. Then, the following properties hold for the two semantics:

\footnotesize
$$ \llbracket \varphi \rrbracket (s) > 0 \Longrightarrow   \mathcal{Q}, s  \models \varphi $$
$$ \llbracket \varphi \rrbracket (s) < 0 \Longrightarrow   \mathcal{Q}, s  \not \models \varphi $$
\normalsize
\end{theorem}

\begin{proof} The proof can be derived by structural induction on the operational semantics. 
\end{proof}

\begin{remark} Theorem~\ref{theorem1} provides the basis of the 
techniques for pattern generation discussed in the following sections.
It is worth to note that, in the case $ \llbracket \varphi \rrbracket (s) = 0$, it
is not possible to infer whether $Q$ violates or satisfies a TSSL formula $\varphi$  and
only in this particular case we
need to resort to the qualitative semantics for determining it.
\end{remark}

%%%%%%%%%%%%%%%%%%%%%%%%%%%%%%%%%%%%%%%%%%%%%%%%%%%%%%%%%%%%%%%%%%%%%%%%%%%%%%%%
%%%%%%%%%%%%%%%%%%%%%%%%%%%%%%%%%%%%%%%%%%%%%%%%%%%%%%%%%%%%%%%%%%%%%%%%%%%%%%%%

\newcommand{\dataset}{{\cal L}}

\section{TSSL PATTERN CLASSIFIERS}\label{sec:learning}

A QTS can be seen in the context of multi-resolution representation, since the nodes that appear at deeper levels provide information for higher resolutions. Therefore, a TSSL formula can effectively capture properties of an image. However, it is difficult to write a formula that describes a desired property, such as a pattern. Here, we propose to use machine-learning techniques to find such a formula from given sets of positive ($\vec Y_{\pattern}$) and negative ($\vec Y_{\nopattern}$) examples.

We first define a labeled data set from the given data sets  $\vec Y_{\pattern}$ and $\vec Y_{\nopattern}$ as 
\[ \dataset = \{ (\cal{Q}_{\vec y}, \pattern) \mid \vec y \in  Y_{\pattern}\}  \cup  \{ (\cal{Q}_{\vec y}, \nopattern) \mid \vec y \in  Y_{\nopattern}\},\]
where $\cal{Q}_{\vec y}$ is the QTS generated from $\vec y$. Then, we separate the data set $\dataset$ into disjoint training and testing sets $\dataset_{L}, \dataset_{T}$.  
In machine-learning, the training set is used to learn a classifier for a target class, \emph{e.g.} $\pattern$,  and the testing set is used to measure the accuracy of the classifier.  We employ RIPPER~\cite{Cohen95fasteffective}, a rule based learner, to learn a classifier from $\dataset_{L}$, and then translate the classifier into a TSSL formula characterizing $\pattern$. Each rule obtained from the learning algorithm is described as \[r_i : C_i \Rightarrow \sim_i,\] where $C_i$ is a boolean formula over  linear predicates over the variables of the states of a QTS, {\em e.g.} $\mbox{[}s\mbox{]}(m) > d$, and $\sim_i$ takes values from the label set $\{\pattern, \nopattern\}$. A linear predicate for a state $s \in S$ can be written as a TSSL formula via the QTS path from the root $s_\iota$ to $s$. Therefore, each $C_i$ can be translated into an equivalent TSSL formula $\Phi_i$. 
The classification rules are interpreted as nested if-else statements. Hence, a logically equivalent TSSL formula for the desired property is defined as follows:
\begin{equation}\label{eq:learnedFormula}
	\Phi_{\pattern} : =  \bigvee_{j \in R_\pattern} \left ( \Phi_j \wedge  \bigwedge_{i=1,\ldots,j-1} \neg  \Phi_i  \right ),
\end{equation}
where $R_\pattern$ is the set of indices of rules $r_i$ with $\sim_i = \pattern$, and $\Phi_i$ is the TSSL formula obtained from $C_i$.

\begin{figure}[h!]
   \centering
   \includegraphics[width=8.5cm,natwidth=8.5cm,natheight=6cm]{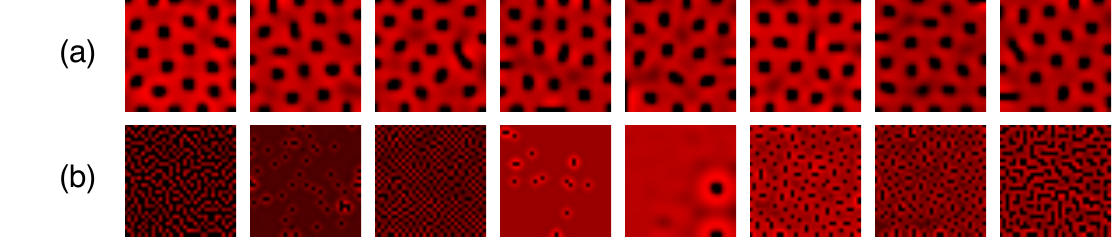}
   \caption{Sample sets of images from the sets (a) $\vec Y^{(1)}_{\pattern}$  and (b) $\vec Y^{(1)}_{\nopattern}$ for the LS pattern.}
\label{fig:traningSet}
\end{figure}

\begin{example}\label{ex:learning} \textbf{LS pattern.} For the LS pattern from Example~\ref{ex:system}, we generate a data set $\vec Y^{(1)}_{\pattern}$ containing $8000$ positive examples by simulating the reaction-diffusion system~\eqref{eq:rd_system_p_ex} from random initial conditions with parameters $\vec R$ and $\vec D_1$. Similarly, to generate the data set $\vec Y^{(1)}_{\nopattern}$ containing $8000$ negative examples, we simulate system~\eqref{eq:rd_system_p_ex} from random initial conditions. However, in this case we use $\vec R$ and randomly choose the diffusion coefficients from $\mathbb{R}_{[0,30]}^2$. As stated before, we only consider the observation of a system in steady-state, for this reason, simulated trajectories that do not reach steady state-in $60$ time units are discarded.  A sample set of images from the sets $\vec Y^{(1)}_{\pattern}$ and $\vec Y^{(1)}_{\nopattern}$ is shown in Figure~\ref{fig:traningSet}. We generate a labeled set $\dataset^{(1)}$ of QTS from these sets, and separate $\dataset^{(1)}$ into $\dataset^{(1)}_{L}, \dataset^{(1)}_{T}$.  
We use RIPPER algorithm implemented in Weka~\cite{WEKA} to learn a classifier from  $\dataset^{(1)}_{L}$. The learning step took $228.5sec$ on an iMac with a Intel Core i5 processor at 2.8GHz with 8GB of memory. The classifier consists of $24$ rules.  The first rule is
\begin{align}
r_1 :& (R \geq 0.59) \wedge (R  \leq 0.70)  \wedge (R.NW.NW.NW.SE \leq 0.75)\wedge  \nonumber \\
&  (R.NW.NW.NW.NW \geq 0.45)  \Rightarrow \pattern \nonumber ,
\end{align}
where $R$ denotes the root of a QTS, and the labels of the children are explained in Figure~\ref{fig:quadtree}.  
Rule $r_1$ translates to the following TSSL formula:
\footnotesize
\begin{align}
\Phi_1 : & (m \geq 0.59) \wedge (m \leq 0.70) \wedge ( \exists_{NW}  \bigcirc \exists_{NW}  \bigcirc \exists_{NW}  \bigcirc \exists_{SE}  \bigcirc m \geq  0.75 ) \wedge  \nonumber \\& ( \exists_{NW}  \bigcirc \exists_{NW}  \bigcirc \exists_{NW}  \bigcirc \exists_{NW}  \bigcirc m \geq  0.45 ). \nonumber
\end{align}
\normalsize

We define the TSSL formula $\Phi^{(1)}_{\pattern}$ characterizing the pattern as in~\eqref{eq:learnedFormula}, and model check QTSs from $\dataset^{(1)}_{T}$ $(|\dataset^{(1)}_{T}| = 8000)$ against  $\Phi^{(1)}_{\pattern}$, which yields  a high prediction accuracy ($96.11\% $) with $311$ miss-classified QTSs.

\textbf{FP and SS patterns.} We follow the above explained steps to generate data sets $\vec Y^{(i)}_{\pattern}, \vec Y^{(i)}_{\nopattern}$, generate labeled data sets $\dataset^{(i)}_{L}, \dataset^{(i)}_{T}$, and finally learn formulas $\Phi^{(i)}_{\pattern}$ for the FP and SS patterns corresponding to diffusion coefficient vectors $\vec D_i$, $i=2,3$ from Example~\ref{ex:system}. Due to the space limitations, we only present the results on the test sets.  The model checking of the QTSs from the corresponding test sets yields high prediction accuracies $98.01\%$, and $93.13\%$ for $\Phi^{(2)}_{\pattern}$, and $\Phi^{(3)}_{\pattern}$, respectively. 
\end{example}

%%%%%%%%%%%%%%%%%%%%%%%%%%%%%%%%%%%%%%%%%%%%%%%%%%%%%%%%%%%%%%%%%%%%%%%%%%%%%%%%
%%%%%%%%%%%%%%%%%%%%%%%%%%%%%%%%%%%%%%%%%%%%%%%%%%%%%%%%%%%%%%%%%%%%%%%%%%%%%%%%

\section{PARAMETER SYNTHESIS  FOR \\ PATTERN GENERATION}\label{sec:design}

In this section we present the solution to  Problem~\ref{prob:main}, \emph{i.e.} 
a framework to synthesize parameters $\paramSet \in \cal P$ of a reaction-diffusion system $\vec S$~\eqref{eq:rd_system} such that the observations of system $\systemP[\paramSet]$ satisfy  a given TSSL formula $\Phi$.
First, we show that the parameters of a reaction-diffusion system that produce trajectories satisfying the TSSL formula can be found by optimizing quantitative model checking results. Second, we include the optimization in a supervised iterative procedure for parameter synthesis. 

We slightly abuse the terminology and say that a trajectory $\vec x(t), t \geq 0$ of system $\systemP[\paramSet]$ satisfies $\Phi$ if the QTS $\mathcal{Q} = ( S, s_\iota, \tau, \Sigma, [.] , L)$ of the corresponding observation, $H(\vec x(\bar t))$, satisfies $\Phi$, \emph{i.e} $\mathcal{Q}  \models \Phi$, or $\llbracket \Phi \rrbracket (s_\iota) > 0$. 

We first define an induced quantitative valuation of a system $\systemP[\paramSet]$ and a set of initial conditions $\cal{X}_0$ from a TSSL formula $\Phi$ as:
\scriptsize
\begin{equation}\label{eq:inducedQuantitative}
	\llbracket \Phi \rrbracket (\systemP[\paramSet]) = \min_{x_0 \in \cal{X}_0} \{ \llbracket \Phi \rrbracket (s_\iota)  \mid \mathcal{Q} = ( S, s_\iota, \tau, \Sigma, [.] , L) \text{ is QTS of } H(\vec x(\bar t)), \vec x(0) = x_0 \}
\end{equation}
\normalsize
The definition of the induced valuation of a system $\systemP[\paramSet]$ implies that all trajectories of $\systemP[\paramSet]$ originating from $\cal{X}_0$ satisfy $\Phi$ if
$\llbracket \Phi \rrbracket (\systemP[\paramSet]) > 0.$
Therefore, it is sufficient to find $\paramSet$ that maximizes~\eqref{eq:inducedQuantitative}. 
It is assumed that the ranges $\cal{P} = \cal P_1 \times \ldots \times \cal P_P$ of the design parameters are known. Therefore, the parameters maximizing~\eqref{eq:inducedQuantitative} can be found with a greedy search on a quantization of $\cal{P}$. However, the computation of $\llbracket \Phi \rrbracket (\systemP[\paramSet])$ for a given $\paramSet \in \cal{P}$ is expensive, since it requires to perform the following steps for each $x_0 \in \cal{X}_0$: simulating the system $\systemP[\paramSet]$ from $x_0$, generating QTS $\cal{Q}$ of the corresponding observation, and quantitative model checking of $\cal{Q}$ against $\Phi$. Here, we use the particle swarm optimization (PSO) algorithm~\cite{KennedyEberhartPSO} over $\cal{P}$ with~\eqref{eq:inducedQuantitative} as the fitness function. The choice of PSO is  motivated by its inherent distributed nature, and its ability to operate on irregular search spaces. In particular, PSO does not require a differentiable fitness function. 

\begin{example} \label{ex:design} \textbf{LS pattern.}
We consider the reaction-diffusion system from Example~\ref{ex:system} and the TSSL formula $\Phi^{(1)}_{+}$ corresponding to the LS pattern from Example~\ref{ex:learning}. We assume that the parameters of the local dynamics are known, $\vec R = [1,-12,-1,16]$, and the diffusion coefficients $D_1$ and $D_2$ are set as the design parameters with $\cal{P} = \mathbb{R}^2_{[0,30]}$. We implement PSO to find $\paramSet \in \cal{P}$ maximizing the induced valuation~\eqref{eq:inducedQuantitative}. 
The PSO computation was distributed on $16$ processors at 2.1GHz on a cluster, and the running time was around $18$ minutes. The optimized parameters are $D_1 = 2.25$ and $D_2 = 29.42$, and the valuation of the system is $0.0023$. A set of observations obtained by simulating~$\systemP[{[2.25, 29.42]}]$ is shown in Figure~\ref{fig:learningObservations}-(a). Note that, while all the observations have some spatial periodicity indicating the presence of a pattern, they are still different from the desired LS pattern. 

\textbf{FP and SS patterns.} We also apply the PSO algorithm on the same setting explained above to maximize the induced valuation~\eqref{eq:inducedQuantitative} for the TSSL formulas  $\Phi^{(2)}_{\pattern}$ (FP pattern) and $\Phi^{(3)}_{\pattern}$ (SS pattern) from Example~\ref{ex:learning}.  The optimized parameters are $[0.083, 11.58]$ and $[1.75, 7.75]$ for $\Phi^{(2)}_{\pattern}$ and $\Phi^{(3)}_{\pattern}$, respectively. Sets of observations obtained by simulating systems  $\systemP[{[0.083, 11.58]}]$ and $\systemP[{[1.75, 7.75]}]$ are shown in Figure~\ref{fig:learningObservationsFP_SS}. In contrast with the LS pattern, the observations are similar to the ones from the corresponding data sets \emph{i.e.}  $\vec Y^{(2)}_{\pattern}$ and $\vec Y^{(3)}_{\pattern}$. 
\end{example}

\begin{figure}[h!]
   \centering
   \includegraphics[width=8.5cm,natwidth=8.5cm,natheight=2cm]{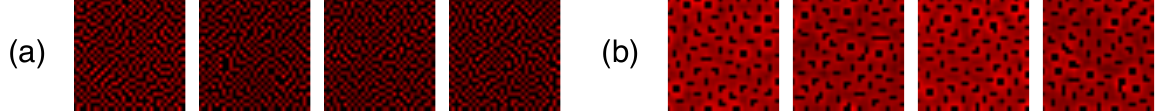}
   \caption{Sample set of observations obtained by simulating (a) $\systemP[{[0.083, 11.58]}]$ and (b) $\systemP[{[1.75, 7.75]}]$.}
\label{fig:learningObservationsFP_SS}
\end{figure}

\begin{figure}[h!]
   \centering
   \includegraphics[width=8.5cm,natwidth=8.5cm,natheight=6cm]{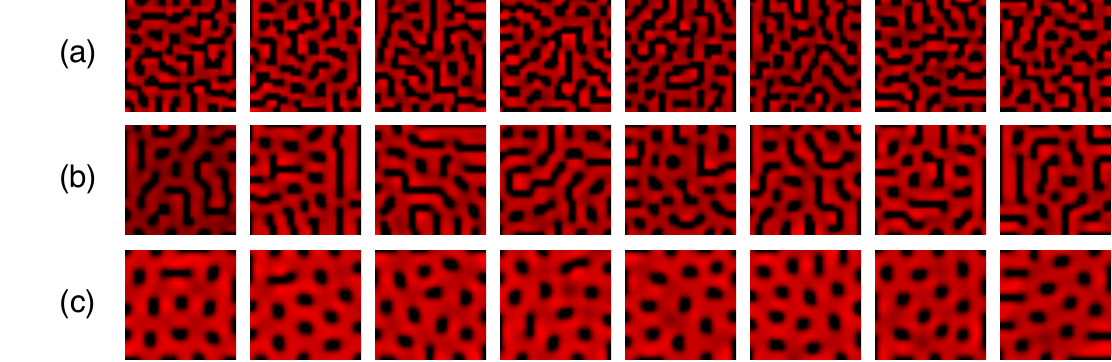}
   \caption{Sample set of observations obtained by simulating (a) $\systemP[{[2.25, 29.42]}]$, (b) $\systemP[{[3.75, 28.75]}]$, and (c) $\systemP[{[6.25, 29.42]}]$.}
\label{fig:learningObservations}
\end{figure}

\begin{remark}
In this paper, we consider the observations generated from a given set of initial conditions $\cal{X}_0$. However, the initial condition can be set as a design parameter and optimized in PSO over a given domain $\mathbb{R}_{[a,b]}^{K\times K \times N}$.
\end{remark}

As seen in Example \ref{ex:design}, it is possible that simulations of the system corresponding to optimized parameters do not necessarily lead to desired patterns. This should not be unexpected, as the formula reflects the original training set of positive and negative examples, and was not ``aware" that these new simulations are not good patterns. A natural extension of our method should allow to add the newly obtained simulations to the negative training set, and to reiterate the whole procedure. This approach is summarized in the \textsc{InteractiveDesign} algorithm.

\begin{table}[!ht]
\label{list:algorithm2}
\begin{tabular*}{8.5cm}{rl}
\hline
\noalign{\vskip 0.3em}
\multicolumn{2}{l}{\textbf{Algorithm} \textsc{InteractiveDesign}} \\
\noalign{\vskip 0.3em}
\hline
\noalign{\vskip 0.3em}
\multicolumn{2}{l}{\footnotesize{\textbf{Input}:} {\hspace*{0.2cm}}\scriptsize{Parametric reaction-diffusion system $\vec S$, ranges of parameters $\cal{P}$,}}\\ 
\multicolumn{2}{l}{ {\hspace*{1.08cm}}\scriptsize{a set of initial states $\cal{X}_0$, sets of observations $\vec Y_{\pattern}$ and $\vec Y_{\nopattern}$}}\\ 
\multicolumn{2}{l}{\footnotesize{\textbf{Output}:}~\scriptsize{Optimized parameters $\paramSet$, the corresponding valuation $\gamma$  }} \\
\multicolumn{2}{l}{{\hspace*{1.0cm}}~\scriptsize{(no solution if $\gamma < 0$)}} \\
\noalign{\vskip 0.5em}
1: & \verb| |\footnotesize{\textbf{while} $True$ \textbf{do}} \\
2: & \verb|   |\footnotesize{$\Phi = Learning(\vec Y_{\pattern}, \vec Y_{\nopattern})$ }\\
3: & \verb|   |\footnotesize{$\{ \paramSet , \gamma \} = Optimization(\vec  S, \cal{X}_0, \Phi)$  \hspace*{0.4cm}  } \\
\multicolumn{2}{l}{{\hspace*{3.3cm}}~\footnotesize{$\triangleright$ $\gamma$ is the induced valuation of $  \systemP[\paramSet]$ }} \\
4: &  \verb|   |\footnotesize{\textbf{if} $\gamma < 0$ \textbf{then return } $\paramSet, \gamma$} \\%\textbf{end if} \\

5: &  \verb|   |\footnotesize{\textbf{end if}} \\
6: &  \verb|   |\footnotesize{UserQuery: Show observations of trajectories }\\
\multicolumn{2}{l}{{\hspace*{3.3cm}}~\footnotesize{of $\systemP[\paramSet]$ originating from $\cal{X}_0$.}} \\
7: &  \verb|   |\footnotesize{\textbf{if} User approves \textbf{then return } $\paramSet, \gamma$} \\%\textbf{end if} \\
8: &  \verb|   |\footnotesize{\textbf{else}}\\
9: &  \verb|      |\footnotesize{$\vec Y_{\nopattern} = \vec Y_{\nopattern} \cup \{H(x(\bar t) ) \mid x(t),t\geq 0,$}\\
\multicolumn{2}{l}{{\hspace*{3.3cm}}~\footnotesize{$ \text{ is generated by }\systemP[\paramSet], x(0) \in \cal{X}_0\}.$}} \\
10: &  \verb|   |\footnotesize{\textbf{end if} }\\
11: & \verb| |\footnotesize{\textbf{end while}} \\

\noalign{\vskip 0.3em}
\hline
\end{tabular*}
\end{table}

We start with the user defined sets of observations $\vec Y_{\pattern}$ and $\vec Y_{\nopattern}$, and learn a TSSL formula $\Phi$ from the QTS representations of the observations (Section~\ref{sec:learning}). Then, in the optimization step, we find a set of parameters $\paramSet$ that maximizes $\gamma = \llbracket \Phi \rrbracket (\systemP[\paramSet])$. If $\gamma < 0$, then we terminate the algorithm as parameters producing observations similar to the ones from the set $\vec Y_{\pattern}$ with respect to the TSSL formula $\Phi$ could not be found. If $\gamma \geq 0$, then the observations of system  $\systemP[\paramSet]$ satisfy $\Phi$.  Finally, the user inspects the observations generated from the reaction-diffusion system with the optimized set of parameters $\systemP[\paramSet]$. If the observations are similar to the ones from the set $\vec Y_{\pattern}$, then we find a solution. If, however, the user decides that the observations do not contain the pattern, then we add observations obtained from system $\systemP[\paramSet]$ to $\vec Y_{\nopattern}$, and repeat the process, \emph{i.e} learn a new formula, run the optimization until the user terminates the process or the optimization step fails ($\gamma < 0$).

\begin{example}\textbf{LS pattern.}
We apply \textsc{InteractiveDesign} algorithm to the system from Example~\ref{ex:design}. A sample set of observations obtained in the first iteration is shown in Figure~\ref{fig:learningObservations}-(a). We decide that these observations are not similar to the ones from the set $\vec Y^{(1)}_{\pattern}$ shown in Figure~\ref{fig:traningSet}-(a), and add these $250$ observations generated with the optimized parameters to $\vec Y^{(1)}_{\nopattern}$ (line 9). In the second iteration, the optimized parameters are $D_1 = 3.75$ and $D_2 = 28.75$, and the observations obtained by simulating~$\systemP[{[3.75, 28.75]}]$ are shown in Figure~\ref{fig:learningObservations}-(b). We continue by adding these to $\vec Y^{(1)}_{\nopattern}$. The parameters computed in the third iteration are $D_1 = 6.25$ and $D_2 = 29.42$. The observations obtained by simulating $\systemP[{[6.25, 29.42]}]$ are shown in Figure~\ref{fig:learningObservations}-(c). % As these observations are similar to the ones from $\vec Y_{\pattern}$, we terminate the algorithm. 
Although the optimized parameters are different from $\vec D_1$, which was used to generate $\vec Y^{(1)}_{\pattern}$, the observations of $\systemP[{[6.25, 29.42]}]$ are similar to the ones from the set $\vec Y^{(1)}_{\pattern}$ and we terminate the algorithm.
\end{example}

%%%%%%%%%%%%%%%%%%%%%%%%%%%%%%%%%%%%%%%%%%%%%%%%%%%%%%%%%%%%%%%%%%%%%%%%%%%%%%%%
%%%%%%%%%%%%%%%%%%%%%%%%%%%%%%%%%%%%%%%%%%%%%%%%%%%%%%%%%%%%%%%%%%%%%%%%%%%%%%%%

\section{CONCLUSION AND FUTURE WORK}
\label{sec:conclusion}

We defined a tree spatial superposition logic (TSSL) whose 
semantics is naturally interpreted over quad trees of partitioned 
images. We showed that formulas in this logic can be efficiently 
learned from positive and negative examples. 
We defined a quantitative semantics for TSSL and combined with 
an optimization algorithm to develop a supervised, iterative procedure 
for synthesis of pattern-producing parameters.    

While the experiments show that the current version of the logic works 
quite well and can accommodate translational and rotational symmetries 
commonly found in biology patterns, there are several directions of future work. 
First, we expect that even better results could be obtained if more statistical 
moments were used, rather than just the mean as in the current version of 
this work. Second, we do not exploit the full semantics of the logic in this paper. 
In future work, we plan to investigate reasoning about multiple branches and 
using the ``until" operator. Third, we plan to apply this method to more realistic 
networks, such as populations of locally interacting engineered cells. We expect 
that experimental techniques from synthetic biology can be used to ``tune" existing 
synthetic gene circuits to produce global desired patterns. 

%%%%%%%%%%%%%%%%%%%%%%%%%%%%%%%%%%%%%%%%%%%%%%%%%%%%%%%%%%%%%%%%%%%%%%%%%%%%%%%%
%%%%%%%%%%%%%%%%%%%%%%%%%%%%%%%%%%%%%%%%%%%%%%%%%%%%%%%%%%%%%%%%%%%%%%%%%%%%%%%%


%\begin{thebibliography}{99}
%
%\bibitem{c1}
%J.G.F. Francis, The QR Transformation I, {\it Comput. J.}, vol. 4, 1961, pp 265-271.
%
%\bibitem{c2}
%H. Kwakernaak and R. Sivan, {\it Modern Signals and Systems}, Prentice Hall, Englewood Cliffs, NJ; 1991.
%
%\bibitem{c3}
%D. Boley and R. Maier, "A Parallel QR Algorithm for the Non-Symmetric Eigenvalue Algorithm", {\it in Third SIAM Conference on Applied Linear Algebra}, Madison, WI, 1988, pp. A20.
%
%\end{thebibliography}

\end{document}